\documentclass[11pt]{article}

\usepackage[final]{acl}

\usepackage{times}
\usepackage{latexsym}
\usepackage{amsmath}
\usepackage{amsfonts}
\usepackage{subcaption}
\usepackage{booktabs}
\usepackage{array}
\usepackage[most]{tcolorbox}
\usepackage{hyperref}

\usepackage[T1]{fontenc}

\usepackage[utf8]{inputenc}

\usepackage{microtype}

\usepackage{inconsolata}

\usepackage{graphicx}

\usepackage{am  sthm}
\newtheorem{theorem}{Theorem}

\usepackage{multirow}

\title{Pre-Generation Hallucination Detection in Large Language Models via Soft-Target Attention Probing}

\author{Amina Miftakhova \\
  Applied AI Institute \\\And
  Alexey Zaytsev \\
  Applied AI Institute \\}

\begin{document}
\maketitle
\begin{abstract}

Detecting hallucination risk before generation enables abstention, retrieval augmentation, and routing decisions without incurring the cost of decoding.
While prior work has shown that such risk can be estimated from a model's internal representations, existing approaches treat this as binary classification over a single decoded output. We instead formulate it as a risk-estimation problem. Under this formulation, we introduce soft-target supervision based on the empirical answer error rate over stochastically sampled outputs — an estimator we prove to be the unique unbiased minimum-variance estimator of the model's per-prompt error probability under its sampling distribution.
We further adapt attention probing to the pre-generation setting, enabling the detector to selectively aggregate hallucination-relevant prompt representations. Across three question-answering benchmarks and five models, attention probing outperforms linear probing on short-answer tasks. Replacing binary labels with soft-target supervision further and consistently improves detection quality.

Source code: \href{https://anonymous.4open.science/r/soft_targets_for_hallu-A707}{anonymous.4open.science}.

\end{abstract}

\begin{figure*}[t]
    \centering
    \includegraphics[width=\linewidth]{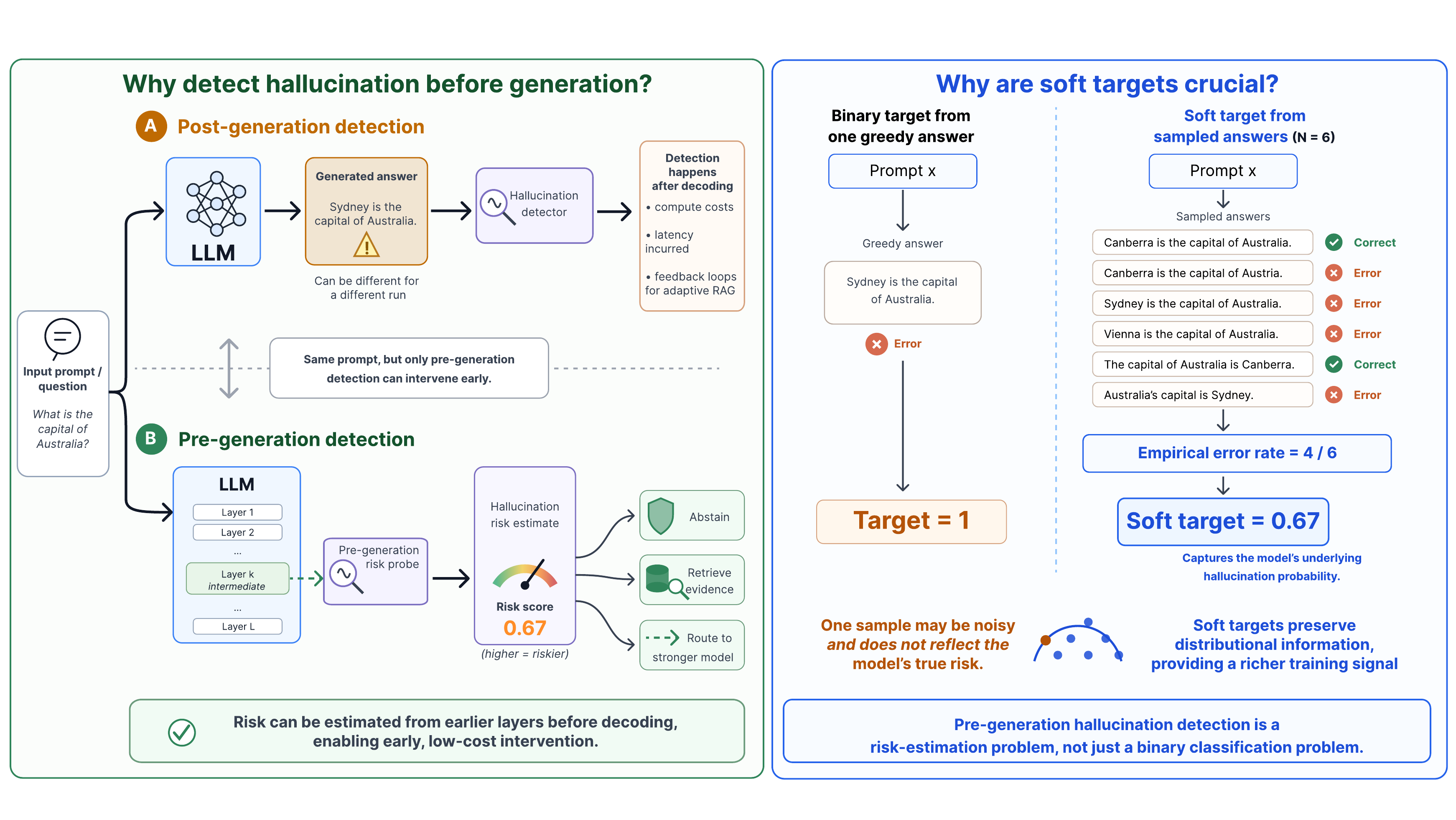}
    \caption{General scheme of pre-generation hallucination detection (left) and soft-target construction (right)}
    \label{fig:big_teaser}
\end{figure*}

\section{Introduction}

Large language models frequently generate confident but factually incorrect responses: a phenomenon known as hallucination~\cite{Huang-Survey}. 
Detecting hallucination risk before generation is especially attractive because it enables abstention, retrieval augmentation, or model routing without paying the full cost of decoding.

Most hallucination detection methods operate post-generation, using uncertainty estimates ~\cite{Malinin2021UncertaintyEI}, semantic consistency across sampled outputs ~\cite{farquhar-se}, or probes over generated representations ~\cite{belinkov-2022-probing,ch-wang-etal-2024-androids}. 
Pre-generation detection has recently received attention as a lightweight alternative ~\cite{alnuhait-etal-2025-factcheckmate, ji-etal-2024-llm, kogilathota-etal-2026-halp}. 
If implemented, it can significantly improve efficiency for adaptive RAG scenarios allowing better selection of sufficient quantity of context ~\cite{moskvoretskii2025adaptive} or even reduce the number of materialized hallucinations ~\cite{alnuhait-etal-2025-factcheckmate}.

However, a key challenge for implementing pre-generation hallucination detector is formulating a training target. 
Existing pre-generation detectors typically assign binary labels derived from a single greedy-decoded answer ~\cite{ji-etal-2024-llm, alnuhait-etal-2025-factcheckmate}. 
In reality, a pre-generation answer correctness appears a random variable: during sampling the same model produces both correct and incorrect answers~\cite{manakul-selfcheckgpt, farquhar-se}, reflecting genuine stochasticity in its output distribution. 
A binary label from a single decoded output is therefore a noisy proxy for the model's underlying error tendency on that input ~\cite{muller-2019-smoothing, guo-2017-calibration}. 

We argue that pre-generation detection is more precisely framed as \emph{risk estimation}: the task is to predict the probability that the model hallucinates on a given prompt under its own answer distribution. 
Under this framing, we propose soft-target supervision derived directly from answer correctness: the training target for each prompt is the empirical fraction of sampled answers that are factually incorrect, estimated at temperature $\tau=1$ to reflect the model's base distribution and provide a consistent difficulty signal across deployment conditions. 
We theoretically and empirically show that sample error rate is an unbiased minimum-variance estimator of the error probability under the model's sampling distribution. 

A second, independent issue concerns representation: how to extract hallucination-relevant signal from hidden states. 
We systematically compare probing approaches in the pre-generation setting, contrasting uniform aggregation with attention probing ~\cite{ch-wang-etal-2024-androids}, which uses a learned query vector to selectively weight token-level hidden states. We further study which layers of the generating model carry the strongest hallucination signal, finding that intermediate layers are consistently most informative, suggesting that reliable risk estimates can be obtained without waiting for the full forward pass to complete.

Our contributions are:
\begin{itemize}
    \item \textbf{Risk-estimation framing and soft-target supervision.} We reframe pre-generation hallucination detection (Figure~\ref{fig:big_teaser}, left) as a risk-estimation problem and introduce soft-target supervision derived from empirical answer error rates (Figure~\ref{fig:big_teaser}, right). We prove that the sample error rate is an unbiased, minimum-variance estimator of the the error probability under the model's sampling distribution.

    \item \textbf{Probing approach and layer analysis.} A systematic comparison of probing architectures shows that attention probing, which uses a learned query vector to selectively weight token-level hidden states, outperforms uniform aggregation in the pre-generation setting. Layerwise analysis further reveals that intermediate layers carry the strongest hallucination signal across models and datasets, thus allowing efficient hallucination detection.


    \item \textbf{Empirical findings of systematic evaluation.} In the pre-generation setting, soft-target supervision provides a consistently stronger training signal than binary labels on open-ended question answering task. We demonstrate that our findings hold across three model families (Qwen, LLaMA, and Gemma) in the sub-9B parameter regime, across architectures of varying scale and design.
\end{itemize}

\section{Related Work}

LLMs can produce fluent text that is factually incorrect or inconsistent with the provided context, a phenomenon commonly referred to as hallucination ~\cite{Huang-Survey, Sahoo2024-survey}. 
Its prevalence in high-stakes settings has motivated a substantial body of work on their detection and mitigation.

\paragraph{Post-generation hallucination detection}

Post-generation methods identify hallucinations after text has been produced and can be broadly grouped into black-box, grey-box, and white-box approaches.

Black-box methods operate on generated text alone: SelfCheckGPT ~\cite{manakul-selfcheckgpt} exploits semantic inconsistency across multiple samples, while LLM-as-a-judge ~\cite{Zheng-2023-llm-judge} approaches use a separate model as an evaluator. 

Grey-box methods additionally access the output distribution. Token-level entropy ~\cite{Malinin2021UncertaintyEI} provides a strong baseline, and semantic entropy ~\cite{farquhar-se} improves on it by clustering sampled responses into equivalence classes and computing entropy over cluster probabilities, yielding higher ROC-AUC on long-form QA.

White-box methods exploit internal model states. A foundational result is that classifiers trained on hidden-layer activations reliably predict whether a statement the model reads or generates is truthful ~\cite{azaria-mitchell-2023-saplma}, showing that veracity is encoded in the model's internal state rather than only in its output distribution. Burns et al. ~\cite{Burns2022DiscoveringLK} push this further with Contrast- Consistent Search (CCS), an unsupervised probe that recovers truth-related directions in activation space by exploiting the consistency constraint that a statement and its negation must be assigned opposite truth values. Linear probes over pooled hidden states ~\cite{belinkov-2022-probing} are the standard supervised baseline. Kossen et al. ~\cite{kossen2025semantic} improve upon them by training lightweight probes to predict semantic entropy directly from hidden states. Orgad et al. ~\cite{orgad2025llms} deepen our understanding of where truthfulness is encoded, finding that the relevant signal is concentrated on specific tokens rather than uniformly distributed, and that error detectors trained this way fail to generalize across datasets, suggesting truthfulness encoding is multifaceted rather than universal. Attention-based methods include Lookback Lens ~\cite{chuang-2024-lookback}, which uses the ratio of attention placed on context versus newly generated tokens, and Xu et al. ~\cite{xu-etal-2023-understanding}, who identify internal hallucination symptoms via contrastive source perturbations. Bazarova et al. ~\cite{toha} detect hallucination through topological divergence between generated text and source context, while RAGLens ~\cite{xiong2026-raglens} trains sparse autoencoders on hidden states to identify interpretable features correlated with unfaithful generation.

\paragraph{Pre-generation hallucination detection}

Pre-generation methods assess hallucination risk from the input representation, before any text is produced. 
This line of work is still relatively recent. 
An early precursor is Kadavath et al. ~\cite{Kadavath2022LanguageM}, who study whether LLMs can estimate the probability that the model knows the answer to a query by examining model self-evaluations. While their approach conditions on a proposed answer and thus is not strictly pre-generative, it motivates the broader question of whether correctness can be predicted before decoding. 
Gottesman and Geva ~\cite{gottesman-geva-2024-estimating} answer this more directly with KEEN, a probe trained on internal representations of subject entities that predicts both QA accuracy and open-ended factuality, without generating a single output token. 
Ji et al. ~\cite{ji-etal-2024-llm} and Alnuhait et al. ~\cite{alnuhait-etal-2025-factcheckmate} train linear probes on input representations with binary labels, establishing strong baselines. 
Kossen et al. ~\cite{kossen2025semantic} show that probes predicting semantic entropy transfer to the pre-generation setting, albeit with reduced accuracy. 
The approach has also been extended to vision-language models: HALP ~\cite{kogilathota-etal-2026-halp} probes pre-generative internal states of VLMs to predict hallucination risk before decoding begins, demonstrating that the principle generalizes across modalities and architectures.

However, the use of binary supervision targets in pre-generation setting is statistically ill-posed: since generation has not yet occurred, hallucination is not a deterministic outcome but a probability over the model's stochastic output distribution. Hard labels collapse this distribution to a single sample, introducing irreducible label noise. Our work addresses this directly by framing pre-generation detection as estimation of hallucination probability, and training probes on soft targets derived from the output distribution.

\section{Methodology}
\label{sec:methodology}

\subsection{Problem Statement}

We study hallucination detection before decoding begins. Given an input prompt $x$, a language model induces a distribution over possible answers $a \sim p_\theta(\cdot \mid x)$. In the pre-generation setting, the quantity of interest is not whether a single sampled answer is hallucinated, but rather the model's hallucination risk under its own answer distribution. We therefore treat hallucination detection as a supervised prediction problem over the model's internal representations, where the target is an estimate of the probability that the model will produce an incorrect answer for the given prompt.

Formally, let $H(x_i) = \{h_j^{(l)}\}_{j=1,\, l=1}^{n, L}$ denote the full set of hidden states produced by the model for all $n$ input tokens across all $L$ layers. The hallucination risk detection model approximates:
\begin{align*}
    p_i^* &= \mathbb{E}_{a \sim p_\theta(\cdot \mid x_i)}\bigl[\mathbf{1} [\text{incorrect}(a)]\bigr] \\
        &= \sum_a p_\theta(a \mid x_i)\,\mathbf{1}[ \text{incorrect}(a)]
\end{align*}
using a chosen subset of $H(x_i)$ to produce $\hat{p}_i^*(H(x_i))$. 
The probe is a layer-agnostic framework: it can be applied to the hidden states of any single layer $l$, and we study the effect of this choice via a layerwise ablation in the Results ~\ref{subs:layer_analysis} section.

\subsection{Target Construction}

\paragraph{Binary targets.} In the binary setting, the target is $y_i \in \{0, 1\}$, derived from the single greedy-decoded answer $a^\dagger = \arg\max_a\, p_\theta(a \mid x_i)$:
$$y_i^{\text{greedy}} = \mathbf{1}[\text{incorrect}(a^\dagger)]$$
Since $y_i^{\text{greedy}}$ is a deterministic constant, its bias relative to the true target $p_i^*$ is:
$$\mathrm{Bias}(y_i^{\text{greedy}}) = \mathbf{1}[\text{incorrect}(a^\dagger)] - p_i^*$$
This mismatch associated with the variance of generation is structural and irreducible. It can be as large as $\pm 1$ in the extreme cases described below.

\begin{enumerate}
    \item \textbf{Underestimation}: the model assigns its single highest probability mass to a correct answer while distributing mass $1 - \epsilon$ across many incorrect ones, giving $y_i^{\text{greedy}} = 0$ but $p_i^* \approx 1$;
    
    \item \textbf{Overestimation}: the model's top answer is an incorrect paraphrase while nearly all probability mass lies on correct variants, giving $y_i^{\text{greedy}} = 1$ but $p_i^* \approx 0$.
\end{enumerate}

In both cases the mode-mean gap $\delta_i = y_i^{\text{greedy}} - p_i^*$ is large. Crucially, $\delta_i$ is question-specific and need not correlate with the true risk $p_i^*$. Thus, a probe trained on greedy targets may learn to predict decoding artifacts rather than the genuine hallucination risk of the model.

\paragraph{Soft targets.} To address this limitation, we derive soft targets from the empirical error rate over $N$ answers sampled via ancestral sampling from the tempered distribution $p_\theta^\tau(\cdot \mid x_i)$:
$$\hat{y}_i = \frac{1}{N}\sum_{j=1}^{N} Z_j, \quad Z_j = \mathbf{1}[\text{incorrect}(a_i^{(j)})]$$
At a fixed temperature $\tau = \hat{\tau}$, this estimator is unbiased for $p_i^*(\hat{\tau}) = \mathbb{E}_{a \sim p_\theta^{\hat{\tau}}}[\mathbf{1}[\text{incorrect}(a)]]$ by linearity of expectation and the independence of samples. Furthermore, $\hat{y}_i$ is a uniformly minimum-variance unbiased estimator (UMVUE) of $p_i^*(\hat{\tau})$; we provide a formal proof further in the Paragraph~\ref{par:soft_target}.

We construct soft targets at $\tau = 1$. At this temperature, $p_\theta^{\tau=1} \equiv p_\theta$, making $p_i^*(1)$ the error probability under the model's original learned distribution and providing the most principled default for soft-target construction, free from any temperature hyperparameter.

\paragraph{The empirical error rate is the UMVUE of the model's error 
probability.}
\label{par:soft_target}

Let $x_i$ be a fixed input and let the model induce a distribution $p_\theta(\cdot \mid x_i)$ at temperature $\tau = 1$, i.e., the raw model distribution. Define the target quantity:
$$p_i^* = \Pr_{a \sim p_\theta(\cdot \mid x_i)}[\mathrm{incorrect}(a)] \in (0, 1).$$
Assuming each answer $a^{(j)}$ is drawn independently from $p_\theta(\cdot \mid x_i)$, the indicators $Z_j = \mathbf{1}[\mathrm{incorrect}(a^{(j)})]$ are i.i.d.\ $\mathrm{Bernoulli}(p_i^*)$, and the soft target is $\hat{y}_i = \frac{1}{N}\sum_{j=1}^N Z_j$.

\textbf{Unbiasedness.} By linearity of expectation,
$$\mathbb{E}[\hat{y}_i] = \frac{1}{N}\sum_{j=1}^N \mathbb{E}[Z_j] = \frac{1}{N}\sum_{j=1}^N p_i^* = p_i^*.$$

\textbf{UMVUE.} Let $T = \sum_{j=1}^N Z_j$. The joint likelihood is
$$\mathcal{L}(p_i^* \mid Z_1,\dots,Z_N) = (p_i^*)^{T}(1-p_i^*)^{N-T},$$
which depends on the data only through $T$. By the Fisher-Neyman factorization theorem ~\cite{lehmann1998theory}, $T$ is sufficient for $p_i^*$. The Binomial family $\{\mathrm{Binomial}(N, p) : p \in (0,1)\}$ is complete, since it forms a one-parameter exponential family with natural parameter $\log\frac{p}{1-p}$ ~\cite{lehmann1998theory}. Since $\hat{y}_i = T/N$ is an unbiased function of a complete sufficient statistic, the Lehmann-Scheff\'e theorem implies that $\hat{y}_i$ is the unique UMVUE of $p_i^*$ ~\cite{lehmann1998theory}. \hfill$\square$

We also emphasize the importance of variance reduction for the speed of convergence for stochastic gradient methods, as the convergence speed is linear in $\frac{1}{\mathrm{Var}(\mathbf{g})}$ with $\mathbf{g}$ being the stochastic gradient, that is proportional to variance of $\hat{y}_i$, which we minimize for convex~\cite{bottou2018optimization} and non-convex~\cite{yan2018unified} optimization problem cases.
From a statistical learning theory perspective, we can also directly measure the benefit of soft-targets of similar order in Appendix~\ref{sec:stl_motivation}.

\subsection{Attention Probing for Pre-Generation Setting}

Attention probing was originally developed by CH-Wang et al.~\cite{ch-wang-etal-2024-androids} for post-generation hallucination detection, where answer-side hidden states are available. We adapt this mechanism to the pre-generation setting by applying it to prompt-side hidden states instead.

Given the token-level hidden states $\{h_j\}_{j=1}^n$ from the input prompt at a fixed layer $l$, a learnable query vector $q \in \mathbb{R}^d$ produces a weighted aggregation:
$$\hat{h}_i = \sum_{j=1}^n \alpha_{ij}\, h_j, \quad \alpha_{ij} = \frac{\exp(q^\top h_j)}{\sum_k \exp(q^\top h_k)}$$
This allows the probe to focus selectively on the prompt positions most informative for hallucination risk, rather than aggregating all token representations uniformly as in linear probing. A logistic regression head then maps the aggregated representation to a hallucination risk score:
$$\hat{p}_i = \sigma(w^\top \hat{h}_i + b)$$
Both the binary and soft-target variants are trained with binary cross-entropy loss, treating the continuous soft target $\hat{y}_i \in [0,1]$ as the supervision signal directly:
$$
    \mathcal{L} = -\frac{1}{M}\sum_{i=1}^{M}\bigl[\hat{y}_i \log \hat{p}_i + (1 - \hat{y}_i)\log(1 - \hat{p}_i)\bigr].
$$

\subsection{Experimental Combinations}

The supervision target construction and probe architecture define distinct axes of design. We evaluate three combinations: linear probing with binary targets, attention probing with binary targets, and attention probing with soft targets. We focus comparisons on attention-based probes because attention probing consistently outperforms linear probing across all settings, which we verify in Section~\ref{sec:results}. Thus, the soft-target contribution is evaluated in the stronger attention-probe setting.

\section{Results and Discussion}
\label{sec:results}

\subsection{Experimental Setting}

\paragraph{Datasets} We evaluate on the subsets of three datasets: SQuAD ~\cite{rajpurkar-etal-2016-squad}, Natural Questions (NQ) ~\cite{Kwiatkowski2019NaturalQA} (short extractive), and HotpotQA ~\cite{yang-etal-2018-hotpotqa} (multi-hop short answer). The details on constructing the subsets are provided in Appendix ~\ref{app:dataset}.
 
\paragraph{Models.} We evaluate on five instruction-tuned models from three architecture families, all in the sub-9B parameter regime: Qwen2.5-3B, Qwen2.5-7B, Qwen3.5-9B, Llama-2-7B, and Gemma-4-E2B. For all models, we use the instruction-tuned variant and apply probes to hidden states extracted during a single forward pass over the input prompt, with no modification to the model weights.
 
\paragraph{Correctness evaluation.} An answer is considered correct if the ground truth is contained in the response, or if ROUGE-L $> 0.6$ and NLI-based semantic similarity $> 0.85$, following the evaluation approach used by ~\cite{Kuhn-semantic}.
 
\paragraph{Baselines} We compare against three pre-generation baselines:

\begin{enumerate}
    \item Question length heuristic;
    \item Zero-shot self-assessment, obtained by prompting the model to report its own confidence without generating an answer;
    \item Linear probe on mean-pooled hidden states from the last layer ~\cite{alnuhait-etal-2025-factcheckmate}.
\end{enumerate}

We additionally report post-generation results for length-normalized entropy, linear probe, and attention probe. 

Full implementation details for all baselines are given in Appendix ~\ref{app:baselines}.

\paragraph{Probe training} All probes are trained for up to 30 epochs with early stopping on validation ROC-AUC. We use the Adam optimizer with binary cross-entropy. Hyperparameters are selected via grid search: learning rate $\in \{10^{-4}, 5{\times}10^{-5}, 10^{-5}\}$, weight decay $\in \{0, 10^{-5}, 10^{-4}\}$, batch size $\in \{8, 16\}$. Probes are trained independently per model and dataset. Soft targets are constructed with $N=10$ samples at temperature $\tau = 1$.

\subsection{Results}

We present results with ROC AUC values for binary error classification across different datasets and models in Table \ref{tab:model_comparison_all_qa} and a critical difference diagram for pre-generation methods in Figure~\ref{fig:cd_pre_generation}.
Detailed analysis is provided in subsequent paragraphs.

\begin{figure*}[t]
    \centering
    \includegraphics[width=0.7\textwidth]{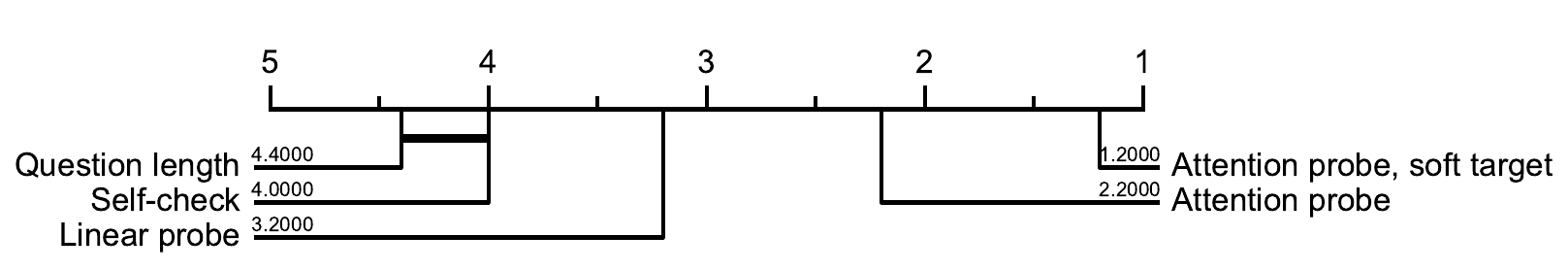}
    \caption{Critical difference diagram for pre-generation hallucination detection methods across dataset--model pairs. Mean ranks over all dataset-model pairs for ROC-AUC of detection are presented, horizontal bars connect not statistically significant differences in ranks for $\alpha = 0.05$.}
    \label{fig:cd_pre_generation}
\end{figure*}

\begin{table*}[t]
\centering
\small
\setlength{\tabcolsep}{4pt}
\begin{tabular}{llccccc}
\toprule
\textbf{Regime} & \textbf{Method} 
& \textbf{Qwen2.5-3B} 
& \textbf{Qwen2.5-7B} 
& \textbf{Qwen3.5-9B} 
& \textbf{Llama-2-7B} 
& \textbf{Gemma-4-E2B} \\
\midrule

\multicolumn{7}{l}{\textbf{SQuAD}} \\
\midrule
\multirow{4}{*}{Post-gen.} 
& Entropy                    & 69.63 & 72.54 & 79.54 & 50.79 & 64.26 \\
& Linear probe               & 58.61 & 68.99 & 76.81 & 79.67 & 73.72 \\
& Attention probe            & 75.32 & \textbf{83.39} & 84.79 & 81.49 & \textbf{79.77} \\
& Attention probe, soft target 
                              & \textbf{80.06} & 83.33 & \textbf{87.54} & \textbf{82.13} & 78.52 \\
\cmidrule(lr){1-7}
\multirow{5}{*}{Pre-gen.}
& Question length            & 52.86 & 56.84 & 51.95 & 55.36 & 51.79 \\
& Self-check, zero-shot      & 57.75 & 51.10 & 57.04 & 52.04 & 51.29 \\
& Linear probe               & 56.45 & 58.64 & 70.27 & 63.99 & 65.37 \\
& Attention probe            & 65.39 & 76.95 & 85.19 & 73.93 & 73.23 \\
& Attention probe, soft target 
                              & \textbf{69.24} & \textbf{80.15} & \textbf{85.66} & \textbf{77.01} & \textbf{78.00} \\

\midrule
\multicolumn{7}{l}{\textbf{Natural Questions}} \\
\midrule
\multirow{4}{*}{Post-gen.}
& Entropy                    & 52.80 & 51.08 & 65.10 & 50.51 & 63.63 \\
& Linear probe               & 64.69 & 70.06 & 71.37 & 66.97 & 69.68 \\
& Attention probe            & 66.62 & 72.64 & 76.61 & \textbf{75.00} & 72.22 \\
& Attention probe, soft target 
                              & \textbf{67.98} & \textbf{74.43} & \textbf{77.28} & 74.36 & \textbf{72.45} \\
\cmidrule(lr){1-7}
\multirow{5}{*}{Pre-gen.}
& Question length            & 50.04 & 53.06 & 52.20 & 51.20 & 53.51 \\
& Self-check, zero-shot      & 57.69 & 59.18 & 60.86 & 47.68 & 56.78 \\
& Linear probe               & 58.56 & 63.41 & 63.87 & 61.46 & 59.38 \\
& Attention probe            & 62.79 & 68.98 & 78.97 & 72.94 & 70.22 \\
& Attention probe, soft target 
                              & \textbf{66.30} & \textbf{71.18} & \textbf{81.52} & \textbf{74.13} & \textbf{71.56} \\

\midrule
\multicolumn{7}{l}{\textbf{HotpotQA}} \\
\midrule
\multirow{4}{*}{Post-gen.}
& Entropy                    & 53.36 & 61.61 & 66.84 & 52.06 & 59.48 \\
& Linear probe               & 65.62 & 72.09 & 67.42 & 67.38 & 68.67 \\
& Attention probe            & 62.82 & 72.72 & \textbf{74.76} & \textbf{72.77} & \textbf{69.31} \\
& Attention probe, soft target 
                              & \textbf{68.09} & \textbf{72.73} & \textbf{74.76} & \textbf{72.77} & \textbf{69.31} \\
\cmidrule(lr){1-7}
\multirow{5}{*}{Pre-gen.}
& Question length            & 50.87 & 59.53 & 50.33 & 51.26 & 52.60 \\
& Self-check, zero-shot      & 54.93 & \textbf{59.30} & 60.73 & 51.21 & 52.80 \\
& Linear probe               & 56.91 & 56.63 & 58.82 & 59.61 & 56.56 \\
& Attention probe            & 54.41 & 58.10 & \textbf{74.84} & 60.13 & \textbf{63.29} \\
& Attention probe, soft target 
                              & \textbf{62.53} & 58.79 & \textbf{74.84} & \textbf{65.48} & \textbf{63.29} \\

\bottomrule
\end{tabular}

\caption{Performance comparison across datasets, models, and hallucination detection methods, measured by ROC-AUC. Results are grouped by dataset and by whether detection is performed in classic setting after generation (Post-gen.) or in our setting before generation (Pre-gen.).}
\label{tab:model_comparison_all_qa}
\end{table*}

Figure~\ref{fig:cd_pre_generation} summarizes the relative performance of the pre-generation hallucination detection methods using a critical difference diagram~\cite{IsmailFawaz2019deep,toha}. 
Each method is ranked within every dataset--model pair according to ROC-AUC, and the reported position corresponds to its average rank across all evaluation pairs. Lower ranks indicate better performance. 
The attention probe trained with soft targets achieves the best average rank, followed by the standard attention probe, other approaches perform substantially worse. 
Moreover, the performance gain between the attention probe with soft targets and other methods is statistically significant.

\paragraph{Pre-generation performance scales with answer determinism.} 
Pre-generation detection performance is strong across diverse datasets. On SQuAD, the difference between the best pre- and post-generation methods is from 10.82 (-13.51\%) to 0.52 (-0.07\%) ROC-AUC points across models (average difference of -4.32\%). On HotpotQA it increases, reaching up to 13.94 (-19.17\%) point difference (-6.55\% on average). 
These results suggest that when answers are short and structurally constrained, a substantial portion of the model's uncertainty is already encoded in the input representations, allowing pre-generation probes to perform competitively with post-generation methods. 
As expected answer length and reasoning complexity increase, uncertainty becomes more dependent on the generation process itself, widening the gap in favour of post-generation detection.

\paragraph{Soft targets improve open-ended tasks in the pre-generation setting.} 
On all three datasets, attention probing trained with soft targets outperforms its binary-target counterpart across all evaluated models (Table ~\ref{tab:model_comparison_all_qa}). The gains are most pronounced on SQuAD, where the ROC-AUC increases from 85.19 to 85.66 (+0.55\%) in the worst case for Qwen3.5-9B and from 73.23 to 78.00 (+6.51\%) in the best case for Gemma-4-E2B. On HotpotQA, in the worst case ROC-AUC does not change for Qwen3.5-7B and Gemma-4-E2B, and in the best case rises from 60.13 to 65.48 (+8.90\%) for Llama-2-7B. However, there are no such consistent improvements in the post-generation setting. These results are consistent with the risk-estimation framing introduced in Section ~\ref{sec:methodology}. Soft targets provide a closer approximation to the model's true per-prompt error probability than a single greedy label, and this additional information is most useful in the pre-generation setting, where the probe must infer distributional uncertainty from input representations alone. We compare the empirical error rate against alternative soft-target formulations in Appendix ~\ref{app:soft_target_formulations}, demonstrating that the error rate consistently performs best, in agreement with its theoretical grounding as an UMVUE of the per-prompt error probability.

\paragraph{Attention probing dominates linear probing in the pre-generation hallucination detection.} The advantage of attention probing over linear is consistent across datasets, where the difference reaches up to 23.1 ROC-AUC points (Table ~\ref{tab:model_comparison_all_qa}). The experimental evidence is consistent with the interpretation that the learned attention mechanism selectively aggregates the prompt positions most informative for hallucination risk, capturing signal that uniform pooling discards. 
On single-token tasks such as multiple-choice or boolean QA, however, as demonstrated in Appendix ~\ref{app:additional_results}, this advantage disappears. The uncertainty signal appears to be sufficiently captured by mean pooling alone, and attention probing provides no additional benefit.

\subsection{Layer Analysis}
\label{subs:layer_analysis}

\begin{figure*}[t]
    \centering
    \includegraphics[width=\linewidth]{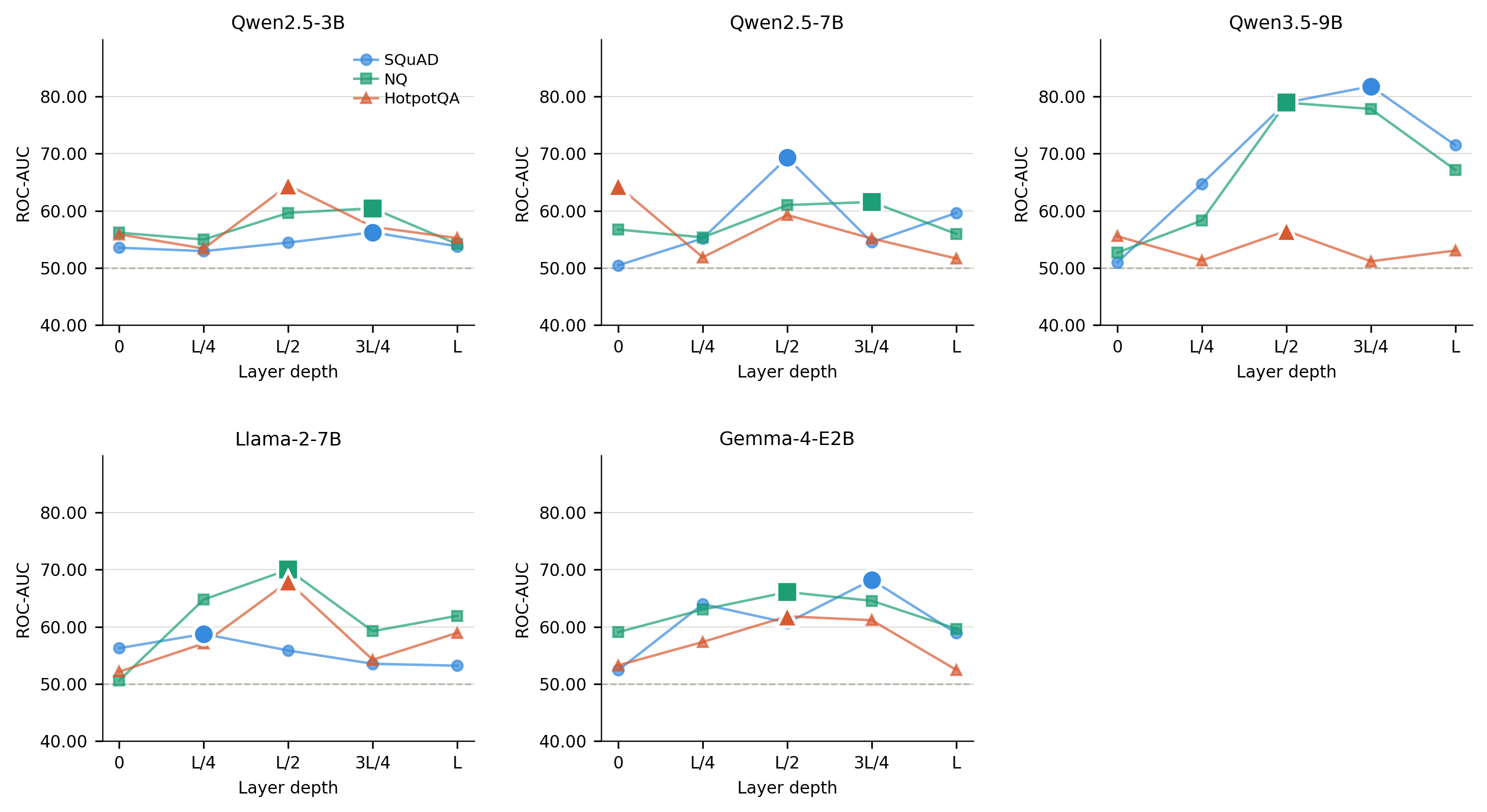}
    \caption{Performance of attention probe with soft targets trained and tested on hidden states from the different layer depth across three datasets. Larger markers indicate the best-performing layer per model and dataset.}
    \label{fig:first}
\end{figure*}

To examine how detection quality varies across the transformer's depth, we train probes using hidden states extracted from a representative subset of layers $\{1,\, \lfloor L/4 \rfloor,\, \lfloor L/2 \rfloor,\, \lfloor 3L/4 \rfloor,\, L\}$ of the generating model. Due to the time and resources constraints, this evaluation was done on the validation and test splits of the data. The results are shown in Figure ~\ref{fig:first}.

Across models and datasets, hallucination-relevant information is distributed throughout the network rather than concentrated in a single layer. The best-performing layer varies by model and task, but intermediate layers are consistently among the most informative, suggesting that useful pre-generation signal emerges well before the final representation and can be extracted without a full forward pass to the last layer.

\subsection{Discussions}

Taken together, the results support three conclusions. First, pre-generation hallucination detection is most reliable when the answer space is short and structurally constrained. On such tasks, the probe's ROC-AUC approaches that of post-generation methods. Second, attention probing is the stronger lightweight architecture in this setting, with gains that are largest precisely where the task provides the most variable prompt structure. Third, soft-target supervision derived from the empirical error rate provides a meaningful and consistent improvement over binary labels on open-ended tasks, with the advantage diminishing as answer determinism increases. This pattern aligns directly with the theoretical motivation in Section~\ref{sec:methodology}.

The primary limitation is the degradation in relative performance as task complexity and answer length increase. On long-context, multi-hop tasks, the detector must infer hallucination risk from more diffuse evidence spread across a longer prompt, and post-generation methods that condition on the generated answer retain a clear advantage. This is not a weakness of the specific probe design but a structural consequence of the pre-generation setting: without access to the generated answer, some uncertainty cannot be resolved. This limitation clarifies the boundary conditions of the method and motivates future work on context-aware probing and integration with abstention or retrieval mechanisms.

\section{Conclusions}

We studied pre-generation hallucination detection through the lens of risk estimation, arguing that the model's per-prompt error probability is a more principled supervision target than a binary label derived from a single greedy output. 
Our central contribution is soft-target supervision based on the empirical answer error rate —-- the unique unbiased minimum-variance estimator of that quantity under the model's sampling distribution. 


Across the open-ended tasks we evaluated, soft-target supervision consistently improves detection quality over binary labels, and the empirical error rate outperforms the alternative soft-target formulations we considered. 
The consistent gains from soft-target training provide a direct empirical evidence that hallucination in the pre-generation setting is fundamentally a distributional property of the model, supporting the risk-estimation view at empirical level.
Our experiments further demonstrate that attention probing is a strong and parameter-efficient mechanism for extracting this risk signal from prompt-side hidden states. Learned attention aggregation consistently outperforms mean-pool probing on short open-ended tasks.
Layerwise analysis shows that hallucination-relevant information is distributed across the network, with intermediate layers frequently among the most informative.
Thus, strong detection signals emerge before the full forward pass completion~\cite{orgad2025llms}.
Finally, when answers are short and structurally constrained, pre-generation probes approach post-generation performance while avoiding the cost of decoding. 
As answer length increases, post-generation methods retain a clear advantage, as uncertainty becomes more dependent on the generation process itself ~\cite{farquhar-se, kossen2025semantic}. 

The broader implication is that pre-generation hallucination detection provides a reliable, interpretable, and earlier signal if augmented with soft-target training. This result, combined with the finding that intermediate layers are frequently the most informative, points toward a practically relevant operating regime: on tasks with short structured answers, reliable hallucination risk estimates can be obtained before generation and partway through the forward pass.

The scope and limitations of these conclusions are discussed in Section ~\ref{sec:limitations}. Natural directions for future work include extending the risk-estimation framework to to exploit the intermediate-layer finding for early-exit inference ~\cite{schuster-2022} and integrating pre-generation risk scores with abstention ~\cite{Kadavath2022LanguageM} or retrieval-augmented generation ~\cite{Asai2023SelfRAGLT} policies.

\section{Limitations}
\label{sec:limitations}

We group the limitations of this work into three categories: theoretical assumptions underlying the framework, empirical scope of the evaluation, and practical constraints on applicability.

\paragraph{Theoretical assumptions.}

\begin{enumerate}
    \item The UMVUE argument assumes that correctness is a well-defined binary property of each answer. For truly open-ended or creative tasks where correctness is graded or subjective, this assumption breaks down and the framework is not directly applicable; the semantic entropy framework ~\cite{farquhar-se} may be more appropriate in such settings.

    \item Soft targets are defined relative to a fixed sampling temperature. The paper evaluates targets constructed at $\tau = 1$, which recovers the model's base distribution, but the probe's performance under a mismatch between training temperature and deployment temperature has not been characterized. Practitioners deploying models at temperatures other than $\tau = 1$ should treat soft-target estimates with appropriate caution.

    \item The soft-target construction assumes a reliable correctness-checking pipeline. Errors in automated evaluation introduce noise into the targets and, if systematic, can bias the soft-target estimator away from the true error probability, violating the unbiasedness guarantee. We use automated evaluation that occasionally produces incorrect judgments. The approach is therefore best suited to tasks where correctness can be verified reliably and should be applied with caution on tasks with ambiguous or graded answers.
\end{enumerate}

\paragraph{Empirical scope.}

\begin{enumerate}
    \setcounter{enumi}{3}
    \item We evaluate models up to 7B parameters due to computational constraints. Larger models may exhibit different layer-wise distributions of hallucination-relevant signal — in particular, the optimal probing layer may shift deeper as model capacity increases — and whether the soft-target and attention-probing findings generalize to that regime remains an open question.

    \item All datasets are in English. Hallucination behaviour may differ across languages, particularly in low-resource settings where the model's output distribution is less calibrated.

    \item The evaluation covers ancestral sampling and greedy decoding. Robustness of the soft-target framework under other decoding strategies (e.g., beam search, nucleus sampling with varying $p$) has not been assessed.

    \item Further experiment would further clarify the usefulness of detected halluciations for their mitigation within mitigation or adaptive RAG frameworks. 
\end{enumerate}

\paragraph{Practical constraints.}

\begin{enumerate}
    \setcounter{enumi}{6}
    \item The method operates in the white-box setting and requires access to intermediate layer activations from a transformer-based model with a token-level residual stream. It is not applicable to black-box API models or to architectures that do not expose token-level hidden states in this form (e.g., state-space models).

    \item Pre-generation detection is inherently constrained by the absence of answer-side evidence. On tasks requiring deep multi-step reasoning or long-form synthesis, uncertainty becomes substantially dependent on the generation process itself, and post-generation methods retain a clear advantage. The appropriate interpretation is that pre-generation detection provides an earlier and cheaper signal that complements rather than replaces post-generation detectors.
\end{enumerate}

\section{Ethics Statement}

\subsection{Use of AI Assistants}

During the preparation of this manuscript, AI-assisted writing tools were used for proofreading and critical review of individual sections. As a result, some parts of the paper may be classified as AI-generated, AI-edited, or a mix of human and AI contributions. All scientific content, experimental design, results, and conclusions are the sole responsibility of the authors.

\bibliography{custom}

\appendix

\section{Statistical learning theory justification for soft-targets}
\label{sec:stl_motivation}

Our results shows that, under
cross-entropy training, empirical soft targets provide an unbiased estimate of
the ideal objective with target \(p^\star(x)\), the conditional hallucination
risk of the model. 
Since the cross-entropy loss is linear in the target,
averaging over sampled generations does not change the population objective;
instead, it reduces the sampling-noise component of the uniform generalization
bound at rate \(O((mN)^{-1/2})\). 
This explains why soft-target supervision
provides a more stable learning signal than a single greedy or sampled binary
label in pre-generation hallucination detection.
The proof and the theorem are natural extensions of existing results for statistical learning theory~\cite{mohri2018foundations}.

\begin{theorem}
\label{thm:soft_targets_ce_generalization}
Let \(X\sim\mathcal{D}\) denote a prompt, and let
\(Y\in\{0,1\}\) indicate whether a sampled answer is incorrect. Define the
conditional hallucination risk
\[
    p^\star(x)=\mathbb{P}(Y=1\mid X=x).
\]
Let \(\mathcal{F}\) be a finite class of predictors
\(f:\mathcal{X}\to[\varepsilon,1-\varepsilon]\), for some
\(\varepsilon\in(0,1/2)\). Given \(m\) prompts
\(x_1,\ldots,x_m\), suppose that for each prompt we draw \(N\) independent
generations
\[
    Y_{ij}\mid X=x_i \sim \mathrm{Bernoulli}(p^\star(x_i)),
\]
with $j = 1, \ldots, N$,
and construct the soft target
\[
    \widehat p_i = \frac{1}{N}\sum_{j=1}^N Y_{ij}.
\]
Consider the empirical soft-target cross-entropy objective
\[
    \widehat R_N(f)
    =
    \frac{1}{m}\sum_{i=1}^m
    \ell(f(x_i),\widehat p_i),
\]
where
\[
    \ell(q,p)
    =
    -p\log q -(1-p)\log(1-q).
\]
Let the population risk be
\[
    R(f)
    =
    \mathbb{E}_{X\sim\mathcal{D}}
    \bigl[
        \ell(f(X),p^\star(X))
    \bigr].
\]
Then, for every fixed \(f\in\mathcal{F}\),
\[
    \mathbb{E}\!\left[
        \widehat R_N(f)\mid x_1,\ldots,x_m
    \right]
    =
    \frac{1}{m}\sum_{i=1}^m
    \ell(f(x_i),p^\star(x_i)).
\]
Thus, the empirical soft-target cross-entropy is an unbiased estimate of the
ideal cross-entropy objective with target \(p^\star(x)\).

Moreover, with probability at least \(1-\delta\), uniformly over
\(f\in\mathcal{F}\),
\[
\begin{aligned}
& \left|
    \widehat R_N(f) - R(f)
\right| \\
&\le
    \sup_{f\in\mathcal{F}}
    \left|
        \frac{1}{m}\sum_{i=1}^m
        \ell(f(x_i),p^\star(x_i))
        -
        R(f)
    \right| \\
&\quad +
    \log\!\left(\frac{1-\varepsilon}{\varepsilon}\right)
    \sqrt{
        \frac{\log(2|\mathcal{F}|/\delta)}{2mN}
    } .
\end{aligned}
\]
Consequently, if
\[
    \widehat f_N \in \arg\min_{f\in\mathcal{F}} \widehat R_N(f)
\]
and
\[
    f^\star \in \arg\min_{f\in\mathcal{F}} R(f),
\]
then, with probability at least \(1-\delta\),
\[
\begin{aligned}
    &R(\widehat f_N)-R(f^\star) \\
    &\le
    2 \sup_{f\in\mathcal{F}}
    \left|
        \frac{1}{m}\sum_{i=1}^m
        \ell(f(x_i),p^\star(x_i))
        -
        R(f)
    \right| \\
    &+
    2\log\!\left(\frac{1-\varepsilon}{\varepsilon}\right)
    \sqrt{
        \frac{\log(2|\mathcal{F}|/\delta)}{2mN}
    } .
\end{aligned}
\]
Therefore, increasing the number of sampled generations \(N\) reduces the
sampling-noise component of the learning bound at rate \(O((mN)^{-1/2})\).
The case \(N=1\) corresponds to supervision from a single sampled binary
label.
\end{theorem}

\begin{proof}
For a fixed prompt \(x_i\), the cross-entropy loss is linear in the target:
\[
    \ell(f(x_i),\widehat p_i)
    =
    -\widehat p_i \log f(x_i)
    -
    (1-\widehat p_i)\log(1-f(x_i)).
\]
Since
\[
    \mathbb{E}[\widehat p_i\mid x_i]=p^\star(x_i),
\]
we have
\[
\begin{aligned}
    &\mathbb{E}\!\left[
        \ell(f(x_i),\widehat p_i)
        \mid x_i
    \right] \\
    &=
    -p^\star(x_i)\log f(x_i)
    -
    (1-p^\star(x_i))\log(1-f(x_i)) \\
    &=
    \ell(f(x_i),p^\star(x_i)).
\end{aligned}
\]
Averaging over \(i=1,\ldots,m\) proves the unbiasedness statement.

It remains to control the additional deviation caused by using
\(\widehat p_i\) instead of \(p^\star(x_i)\). For each \(i\),
\[
\begin{aligned}
    &\ell(f(x_i),\widehat p_i)
    -
    \ell(f(x_i),p^\star(x_i)) \\
    &=
    -(\widehat p_i-p^\star(x_i))
    \log f(x_i) \\
    &\quad
    +(\widehat p_i-p^\star(x_i))
    \log(1-f(x_i)) \\
    &=
    (\widehat p_i-p^\star(x_i))
    \log\frac{1-f(x_i)}{f(x_i)} .
\end{aligned}
\]
Because \(f(x_i)\in[\varepsilon,1-\varepsilon]\),
\[
    \left|
        \log\frac{1-f(x_i)}{f(x_i)}
    \right|
    \le
    \log\frac{1-\varepsilon}{\varepsilon}.
\]
Moreover, \(\widehat p_i-p^\star(x_i)\) is the average of \(N\) centered
Bernoulli variables and is therefore sub-Gaussian with variance proxy
\(1/(4N)\). Hence, for any fixed \(f\), Hoeffding's inequality gives
\begin{align*}
    & \left|
        \frac{1}{m}\sum_{i=1}^m
        \left[
            \ell(f(x_i),\widehat p_i)
            -
            \ell(f(x_i),p^\star(x_i))
        \right]
    \right| \\
    & \le
    \log\!\left(\frac{1-\varepsilon}{\varepsilon}\right)
    \sqrt{
        \frac{\log(2/\delta)}{2mN}
    }
\end{align*}
with probability at least \(1-\delta\). Applying a union bound over
\(f\in\mathcal{F}\) yields the stated uniform bound. Combining this sampling
deviation with the standard generalization gap between the empirical ideal
risk and the population risk gives the result.

Finally, the excess-risk bound follows from the standard empirical risk
minimization argument:
\[
    R(\widehat f_N)-R(f^\star)
    \le
    2\sup_{f\in\mathcal{F}}
    \left|
        \widehat R_N(f)-R(f)
    \right|.
\]
\end{proof}

\section{Dataset Construction}
\label{app:dataset}

\subsection{Dataset Selection and Preprocessing}

We evaluate on three QA datasets chosen to span the spectrum of answer formats and reasoning demands. For each dataset, we construct a separate hallucination detection dataset per generating model, as hallucination behaviour is model-dependent.
 
SQuAD ~\cite{rajpurkar-etal-2016-squad} contains over 100,000 extractive reading comprehension questions. Since the proposed methods do not require large training sets and full-scale soft target construction is computationally expensive, we restrict experiments to the official validation split of 10,570 samples, which we randomly partition into training, validation, and test subsets.
 
HotpotQA ~\cite{yang-etal-2018-hotpotqa} contains approximately 113,000 multi-hop questions across three difficulty levels (easy, medium, hard). To construct a balanced evaluation set that controls for difficulty, we sample 3,000 examples from each difficulty level and merge them into a dataset of 9,000 examples, preserving the difficulty distribution across splits.
 
Natural Questions ~\cite{Kwiatkowski2019NaturalQA} contains over 300,000 Wikipedia-sourced questions with both short and long reference answers. Since input paragraphs are often too long for smaller models, we use the long answer as the context and the short answer as the gold target. We sample 9,000 examples uniformly.

\subsection{Dataset Statistics}

\begin{table}[h]
\centering
\begin{tabular}{lccc}
\toprule
\textbf{Dataset} & \textbf{Train} & \textbf{Test} & \textbf{Validation} \\
\midrule
SQuAD       & 8456 & 1057 & 1057 \\
HotpotQA    & 6000 & 1500 & 1500 \\
NQ          & 6000 & 1500 & 1500 \\
\bottomrule
\end{tabular}
\caption{Number of samples in each split for the evaluated datasets.}
\label{tab:samples}
\end{table}

Table ~\ref{tab:samples} reports the number of samples in each split across all datasets used in the paper.

\section{Baseline Implementation}
\label{app:baselines}

\subsection{Question and Context Length}

This heuristic baseline tests whether surface-level properties of the input contain hallucination-relevant information. We use the number of whitespace-tokenized words in the concatenated question and context string. No hidden-state extraction is required.

\subsection{Zero-Shot Self-Check}

This baseline elicits the model's verbalized confidence estimate by prompting it directly, before any answer is generated. The prompt used is specified in Appendix ~\ref{app:prompts}.

Rather than taking a hard binary prediction from the first generated token, we extract a soft hallucination probability directly from the model's next-token distribution. Specifically, let $\mathcal{V}_{\text{no}} = \{\texttt{no}, \texttt{No}, \texttt{NO}\}$ and $\mathcal{V}_{\text{yes}} = \{\texttt{yes}, \texttt{Yes}, \texttt{YES}\}$ denote the sets of case variants for each response. The hallucination risk score is defined as the probability mass on negative responses, normalized over the union of both sets.

This produces a continuous score in $[0, 1]$ rather than a hard binary decision, utilizing the full probability distribution over the relevant token surface forms. No training is performed, this is a purely zero-shot signal. The baseline tests whether the model possesses an accessible internal notion of answerability that can be elicited through prompting alone.

\subsection{Length-Normalized Entropy}

Following Malinin and Gales ~\cite{Malinin2021UncertaintyEI}, we compute the sequence-level entropy of the model's distribution normalized by the sequence length. For a generated sequence $a = (a_1, \dots, a_m)$, the length-normalized entropy is $\text{LNE}(a \mid x) = -\frac{1}{m} \sum_{t=1}^m \sum_{v=1}^V p_\theta(v \mid x, a_{<t}) \log p_\theta(v \mid x, a_{<t})$, where $V$ is the vocabulary size and $p_\theta(v \mid x, a_{<t})$ is the model's next-token distribution at position $t$. This is computed via a single teacher-forced forward pass over the greedy-decoded answer. Length normalization corrects for the tendency of longer sequences to accumulate higher entropy. The resulting scalar is used directly as a hallucination risk score without any learned component. Per-position entropy is computed from the full softmax distribution using $\log\text{-softmax}$ for numerical stability, avoiding explicit computation of $\exp$ followed by $\log$.

\subsection{Linear Probing on Pooled Input Representations}

Following Alnuhait et al. ~\cite{alnuhait-etal-2025-factcheckmate}, we extract hidden states from the generating model at a fixed transformer layer $l$ for the full input sequence (prompt tokens only, no generated answer). The per-layer hidden state matrix $H^{(l)} \in \mathbb{R}^{n \times d}$ is mean-pooled across the token dimension to produce a single vector $\bar{h}^{(l)} \in \mathbb{R}^d$. A single output layer is trained on top of $\bar{h}^{(l)}$.

\subsection{Attention Probing in the Post-Generation Setting}

The post-generation attention probe follows the implementation described originally by CH-Wang et al. ~\cite{ch-wang-etal-2024-androids}. The hidden states are extracted from the answer sequence rather than the input. Specifically, let $H^{(l)}_{\text{answer}} \in \mathbb{R}^{(m) \times d}$ denote the hidden states at layer $l$ for the answer token sequence of length $m$. The learned query $q$ attends over all $m$ positions:
 
$$\hat{h}_i = \sum_{j=1}^{m} \alpha_{ij}\, h_j^{(l)}, \qquad \alpha_{ij} = \frac{\exp(q^\top h_j^{(l)})}{\sum_k \exp(q^\top h_k^{(l)})}$$
 
The logistic regression head and training procedure are identical to the pre-generation setting.

\section{Prompts}
\label{app:prompts}

\subsection{Generation Prompt}

\begin{tcolorbox}[title=Generation Prompt]
\small
\begin{verbatim}
Give a short, concise answer to the question
using the information from the paragraph.

Context:
{context}

Question:
{question}
\end{verbatim}
\end{tcolorbox}

\subsection{Self-Check Prompt}

\begin{tcolorbox}[title=Self-Check Prompt]
\small
\begin{verbatim}
Answer with a single word: 'YES' or 'NO':
are you able to answer the question correctly
using the information from the context?

Context:
{context}

Question:
{question}
\end{verbatim}
\end{tcolorbox}

\section{Statistical Significance of the Obtained Results}

\begin{figure*}
    \centering
    \includegraphics[width=\linewidth]{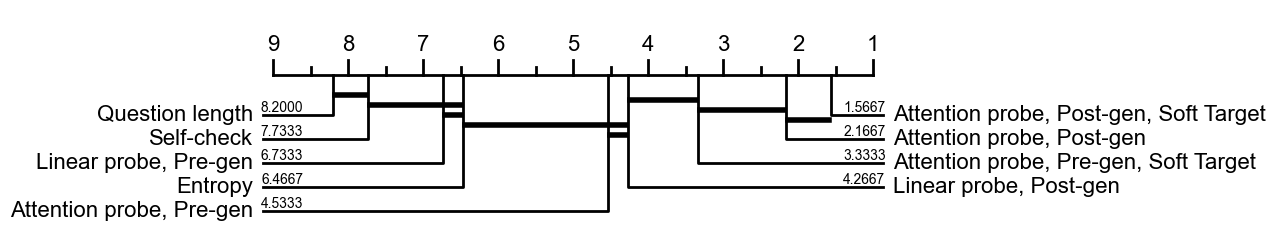}
    \caption{Critical difference diagram for the hallucination detection approaches. The numbers represent the ranks for each method. Thick horizontal lines group methods that are not significantly different.}
    \label{fig:cd_diagram}
\end{figure*}

To assess the statistical significance of the reported differences, we use the critical difference diagram, a computationally efficient non-parametric procedure suitable for comparing multiple classifiers across datasets ~\cite{IsmailFawaz2018deep}. The analysis presented in the Figure ~\ref{fig:cd_diagram} confirms three key findings: 
\begin{itemize}
    \item Pre-generation attention probing trained with soft targets significantly outperforms its binary-target counterpart;
    \item Pre-generation attention probing with soft targets performs comparably to post-generation attention probing trained on binary targets;
    \item Pre-generation attention probing significantly outperforms pre-generation linear probing regardless of supervision target.
\end{itemize}



\section{Soft Target Formulation}
\label{app:soft_target_formulations}

\subsection{Empirical Comparison of Soft Target Formulations}

\begin{table}[h]
\centering
\renewcommand{\arraystretch}{1.3}
\begin{tabular}{p{3cm} >{\centering\arraybackslash}m{1.5cm} >{\centering\arraybackslash}m{2cm}}
\toprule
\textbf{Target} & \textbf{SQuAD} & \textbf{Natural Questions} \\
\midrule
Error rate                                       & \textbf{79.45} & \textbf{74.13} \\
\raggedright Error rate weighted by answer log-prob         &  77.77 & 74.02 \\
\raggedright Relative probability mass of incorrect / correct answers & 66.16 & 73.75 \\
\raggedright Semantic dissimilarity of the greedy answer    & 62.03 & 55.93 \\
\raggedright Mean semantic dissimilarity                      & 54.68 & 58.24 \\
\raggedright Semantic dissimilarity weighted by answer log-prob & 59.31 & 58.90 \\
\bottomrule
\end{tabular}

\caption{Comparison of attention probe trained on different soft target formulations, average ROC-AUC across models}
\label{tab:soft_labels}

\end{table}

All formulations below take as input $N$ answers $\{a^{(j)}\}_{j=1}^N$ sampled i.i.d.\ from $p_\theta(\cdot \mid x_i)$ at temperature $\tau = 1.0$, along with their per-sequence log-probabilities and token counts where required. Let $Z_j = \mathbf{1}[\text{incorrect}(a^{(j)})]$ denote the correctness indicator and let $\tilde{\ell}_j = \log p_\theta(a^{(j)} \mid x_i) / |a^{(j)}|$ denote the per-token log-probability of the $j$-th sample, used to normalize for sequence length.

\paragraph{Error rate.} The simple empirical fraction of incorrect answers: $\hat{y}_i = \frac{1}{N}\sum_{j=1}^N Z_j$. As demonstrated in Section ~\ref{par:soft_target}, it is an unbiased UMVUE for the raw model error probability.

\paragraph{Error rate weighted by answer log-probability.} Samples are reweighted by their softmax-normalized per-token log-probability before aggregation: $\hat{y}_i^{M1} = \sum_{j=1}^N w_j Z_j, \qquad w_j = \frac{\exp(\tilde{\ell}_j)}{\sum_k \exp(\tilde{\ell}_k)}$. Since all samples are drawn from $p_\theta(\cdot \mid x_i)$, the uniform weight $1/N$ is the natural sampling weight. Replacing it with $w_j$ is equivalent to self-normalized importance sampling with the same proposal and target distribution, introducing a systematic bias toward over-representing high-probability incorrect answers.

\paragraph{Relative probability mass of incorrect vs. correct answers.} The log-probability gap between incorrect and correct groups, passed through a sigmoid: $\hat{y}_i^{M2} = \sigma\!\left(\bar{\tilde{\ell}}_{\text{incorr}} -\bar{\tilde{\ell}}_{\text{corr}}\right)$, where $\bar{\tilde{\ell}}_{\text{incorr}}$ and $\bar{\tilde{\ell}}_{\text{corr}}$ are the mean per-token log-probabilities of incorrect and correct samples respectively. This measures how much more probability mass the model concentrates on incorrect answers, rather than the absolute frequency of incorrectness. When all samples fall in one class, the missing group's mean is approximated by the per-token log-probability of the ground-truth answer shifted by a small offset $\epsilon = 0.1 \cdot |\tilde{\ell}_{\text{gt}}|$.

\paragraph{Semantic dissimilarity of the greedy answer.} A single-sample estimate based on the greedy-decoded answer $a^\dagger$: $\hat{y}_i^{M3} = 1 - \text{sim}(a^\dagger,\, g_i)$, where $g_i$ is the gold reference answer and $\text{sim}$ denotes cosine similarity under \textit{nli-roberta-large} sentence embeddings. This requires no sampling and no correctness threshold, but provides only a single-point estimate with $\mathcal{O}(1)$ variance regardless of $N$.

\paragraph{Mean semantic dissimilarity.} The sample mean of per-answer semantic dissimilarity: $\hat{y}_i^{M4} = \frac{1}{N}\sum_{j=1}^N \left(1 - \text{sim}(a^{(j)},\, g_i)\right)$. Unlike the error rate, this is an unbiased estimator of $\mathbb{E}_{a \sim p_\theta}[1 - \text{sim}(a, g_i)]$ rather than of $p_i^*$. The two quantities diverge whenever semantic proximity and correctness are misaligned. For example, a factually wrong answer can be semantically close to the correct one, and a correct paraphrase can be semantically distant from the gold string.

\paragraph{Semantic dissimilarity weighted by answer log-probability.} The log-probability-weighted version of mean semantic dissimilarity:
$\hat{y}_i^{M5} = \sum_{j=1}^N w_j\left(1 - \text{sim}(a^{(j)},\, g_i)\right), \qquad w_j = \frac{\exp(\tilde{\ell}_j)}{\sum_k \exp(\tilde{\ell}_k)}$. This compounds the biases of importance weight distortion and similarity-correctness mismatch, with no cancellation between the two sources of bias in general.

\subsection{Ablation on the Number of Samples}

\begin{table}[h]
\centering
\begin{tabular}{lcc}
\toprule
& \textbf{Llama-2-7B} & \textbf{Qwen2.5-7B} \\
\midrule
$N = 3$  & 57.47 & 52.47 \\
$N = 5$  & 58.17 & 53.82 \\
$N = 10$ & \textbf{59.07} & \textbf{58.26} \\
\bottomrule
\end{tabular}
\caption{Ablation on the number of samples $N$ used to construct soft targets, evaluated on HotpotQA for two representative models, ROC-AUC.}
\label{tab:n_ablation}
\end{table}

Table ~\ref{tab:n_ablation} reports detection quality as a function of the number of samples $N$ used to estimate the per-prompt error rate. Performance increases monotonically with $N$ for both models, consistent with the theoretical expectation: as $N$ grows, the empirical error rate converges to the true error probability $p_i^*$, reducing estimator variance and providing a cleaner supervision signal. While the absolute gains are modest, the trend is consistent across both model families, supporting the use of $N = 10$ as the default in the main experiments.

\section{Computational Cost Analysis}
\label{app:compute_cost}




\begin{figure}
    \centering
    \includegraphics[width=\linewidth]{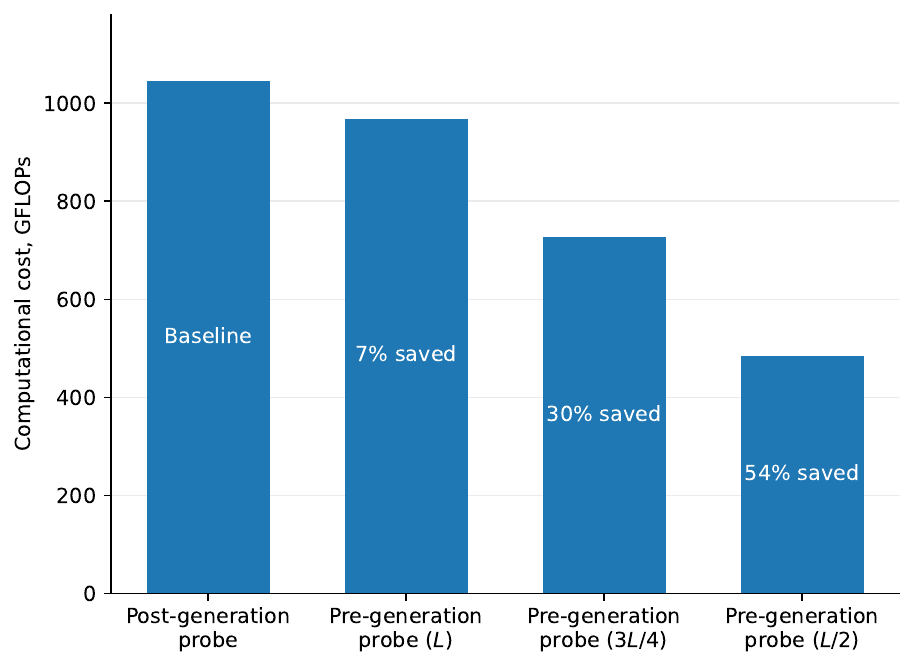}
    \caption{Mean GFLOPs and cost savings relative to pre-generation probing at full depth ($L$) and relative to post-generation probing, averaged across five models at median prompt and response lengths.}
    \label{fig:flops}
\end{figure}

We stated that pre-generation hallucination detection provides the opportunity for early detection. Figure ~\ref{fig:first} demonstrates that in general, middle layers contain the majority of the information about the hallucination, thus suggesting detection on early layers. To quantify the reduction of the inference cost, we estimate the theoretical FLOPs for detection on different layers at median prompt and response lengths from our evaluation datasets and average them across models. Figure ~\ref{fig:flops} demonstrates both absolute FLOPs and relative computational costs savings for different layers.

\section{Results on Additional Datasets}
\label{app:additional_results}

\begin{table*}[t]
\centering
\begin{tabular}{lcc}
\toprule
\textbf{Method / Model} & \textbf{Qwen2.5-3B} & \textbf{Qwen2.5-7B} \\
\midrule
\textbf{Post-generation} \\
\midrule
Entropy           & \textbf{64.99} & 72.32 \\
Linear probe      & 64.03 & 69.18 \\
Attention probe   & 50.00 & \textbf{75.42} \\
\midrule
\textbf{Pre-generation} \\
\midrule
Question length               & 59.21 & 60.15 \\
Self-check, zero-shot         & 53.73 & 60.20 \\
Linear probe                  & \textbf{63.89} & 69.23 \\
Attention probe               & 50.00 & \textbf{74.40} \\
Attention probe, soft target  & 53.55 & 56.95 \\
\bottomrule
\end{tabular}

\caption{Performance comparison across different models and methods on MMLU-Pro dataset, ROC-AUC.}

\label{tab:model_comparison_mmlu-pro}
\end{table*}

\begin{table*}[t]
\centering
\begin{tabular}{lcccc}
\toprule
\textbf{Method / Model} & \textbf{Qwen2.5-3B} & \textbf{Qwen2.5-7B} & \textbf{Llama-2-7B} & \textbf{Mistral-7B} \\
\midrule
\textbf{Post-generation} \\
\midrule
Entropy           & 68.58 & 78.35 & 62.63 & 51.08 \\
Linear probe      & 91.42 & \textbf{94.76} & \textbf{90.66} & 92.24 \\
Attention probe   & \textbf{92.18} & 94.46 & 90.94 & \textbf{93.18} \\
\midrule
\textbf{Pre-generation} \\
\midrule
Question length   & 51.08 & 50.85 & 51.00 & 52.52 \\
Self-check, zero-shot & 63.36 & 64.19 & 52.65 & 61.11 \\
Linear probe          & 69.65 & 82.25 & 68.64 & 78.99 \\
Attention probe       & \textbf{92.03} & \textbf{84.23} & \textbf{89.88} & \textbf{93.21} \\
\bottomrule
\end{tabular}
\caption{Performance comparison across different models and methods on BoolQ dataset, ROC-AUC.}
\label{tab:model_comparison_boolq}
\end{table*}

\begin{table*}[t]
\centering

\begin{tabular}{lccc}
\toprule
\textbf{Method / Model} & \textbf{Llama-2-7B} & \textbf{Llama-2-13B} & \textbf{Mistral-7B} \\
\midrule
\textbf{Post-generation} \\
\midrule
Entropy                        & 62.63 & 52.34 & 66.71 \\
Linear probe                   & 53.63 & 60.71 & 59.84 \\
Attention probe                & 56.37 & 60.41 & 65.31 \\
\midrule
\textbf{Pre-generation} \\
\midrule
Question length        & 55.18 & 53.01 & 52.47 \\
Self-check, zero-shot              & 53.97 & 52.33 & 54.85 \\
Linear probe                       & 53.58 & 61.04 & 59.87 \\
Attention probe                    & \textbf{56.20} & \textbf{67.34} & \textbf{65.24} \\
\bottomrule
\end{tabular}

\caption{Performance comparison across different models and methods on RAGTruth dataset, ROC-AUC.}
\label{tab:model_comparison_ragtruth}
\end{table*}

\paragraph{Datasets.} We report results on three additional datasets: MMLU-Pro ~\cite{mmlu-pro} (10-way multiple choice), BoolQ ~\cite{clark2019boolq} (binary yes/no), and RAGTruth ~\cite{ragtruth} (long-form retrieval-augmented generation). Due to computational constraints, these experiments cover a subset of the models evaluated in the main text; the specific models are indicated in each table.

\paragraph{Scope of evaluation.} We evaluate binary-target probes only on all three datasets. Soft-target supervision is not evaluated here for two reasons. First, MMLU-Pro and BoolQ require single-token answers, and the experiments demonstrated that soft targets provide no benefit, even degrading performance on such tasks (Table ~\ref{tab:model_comparison_mmlu-pro}). Second, RAGTruth does not provide an automated correctness-checking pipeline, which is a prerequisite for soft-target construction.

\paragraph{Results.} Tables ~\ref{tab:model_comparison_mmlu-pro}, ~\ref{tab:model_comparison_boolq}, and ~\ref{tab:model_comparison_ragtruth} report ROC-AUC scores across datasets and models. The results are consistent with two findings from the main text. First, pre-generation attention probing achieves detection quality comparable to post-generation methods across all three datasets, supporting the conclusion that hallucination risk is substantially encoded in the input representations on these tasks. Second, attention probing consistently outperforms linear probing in the pre-generation setting, confirming that learned aggregation provides a more informative representation of prompt-level uncertainty than uniform pooling.


\end{document}